\documentclass{article}
\pdfoutput=1
\usepackage[hide]{ed}  

\usepackage{times}
\usepackage{graphicx}
\usepackage{latexsym}
\usepackage{url}

\usepackage{graphicx}
\usepackage{subfigure}
\usepackage{flushend}
\usepackage{epsfig}
\usepackage{amssymb}
\usepackage{amsmath}

\usepackage{algorithm}
\usepackage{algorithmic}

\newtheorem{theorem}{Theorem}

\newtheorem{proposition}[theorem]{Proposition}

\newcommand{\eat}[1]{}
\newcommand{\OPRA}{{\cal OPRA}}
\newcommand{\opra}[3]{{\scriptscriptstyle #1}\angle_{#2}^{#3}}
\newcommand{\opras}[2]{{\scriptscriptstyle #1}\angle{#2}}
\newcommand{\false}{\mathit{false}}
\renewcommand{\mod}{\,\mathrm{mod}\,}

\begin{document}
\newenvironment{ProofSketch}{\textit{Proof Sketch.}}{\hspace*{\fill}$\Box$}
\newenvironment{proof}{\textit{Proof.}}{\hspace*{\fill}$\Box$}


\title{Qualitative Reasoning about Relative Direction
\\
on Adjustable Levels of Granularity }

\author{$\mbox{Till Mossakowski}^{1}$ and $\mbox{Reinhard Moratz}^2$\\[0.3cm]
$^1\mbox{University of Bremen,}$\\
Collaborative Research Center on Spatial Cognition (SFB/TR 8),\\
Department of Mathematics and Computer Science,\\
and DFKI GmbH,
Enrique-Schmidt-Str. 5, 28359 Bremen, Germany.\\
{\tt Till.Mossakowski@dfki.de}\\[0.3cm]
$^2\mbox{University of Maine,}$\\
Department of Spatial Information Science and Engineering,\\
348 Boardman Hall, Orono, 04469 Maine, USA.\\
{\tt moratz@spatial.maine.edu}\\[0.3cm]
}

\date{}
\maketitle

\begin{abstract}
\begin{sloppy}
  An important issue in Qualitative Spatial Reasoning is the
  representation of relative direction.  In this paper we present
  simple geometric rules that enable reasoning about relative
  direction between oriented points.  This framework, the Oriented
  Point Algebra $\OPRA_m$, has a scalable granularity $m$.  We develop
  a simple algorithm for computing the $\OPRA_m$ composition tables and
  prove its correctness.  Using a composition table, algebraic closure
  for a set of $\OPRA$ statements is sufficient to solve spatial
  navigation tasks.  And it turns out that scalable granularity is
  useful in these navigation tasks.
\end{sloppy}
\end{abstract}
\vspace*{0.5cm}{\bf Keywords:}\\
Qualitative Spatial Reasoning, Constraint-based Reasoning,  Qualitative Simulation 


\section{Introduction}

\ednote{TM: red plot/motivation. Simple direct computation of
composition table}
\noindent The concept of {\it qualitative space} can be characterized by the following quotation
from Galton \cite{Galton00_QualSpatChange}:
\begin{quotation}
\noindent The divisions of qualitative space correspond to salient discontinuities in
our apprehension of quantitative space.
\end{quotation}
If qualitative spatial divisions serve as knowledge representation in a reasoning system
deductive inferences can be realized as constraint-based reasoning \cite{Renz2007}.
An important issue in such Qualitative Spatial Reasoning systems is the representation of 
relative direction\cite{cosyFRE92}, \cite{hernandez_aij}.
Qualitative spatial constraint calculi typically store their spatial knowledge in
a composition table \cite{Renz2007}.
For a recent overview about Qualitative Spatial Reasoning (QSR) we refer to
Renz and Nebel \cite{Renz2007}.

A new qualitative spatial reasoning calculus about relative direction,
the Oriented Point Algebra $\OPRA_m$,
which has a scalable granularity with parameter $m\in \mathbb{N}$ was presented in \cite{Moratz06_ECAI}.
The motivation for this scalable granularity was that
representing relatively fine
distinctions was expected to be useful in more complex navigation tasks. 
It turned out to be difficult to analyze the reasoning rules for this calculus:
The algorithm presented in the
original paper \cite{Moratz06_ECAI} contained many gaps and errors. The algorithm
presented in \cite{FrommbergerEtAl07} is quite lengthy and cumbersome.
\ednote{I do not understand the following argument: ``is too lenghty to be used for
proofs about existence of realizations of abstract configuration/correctness proofs.''}

The paper is organized as follows:
we will first give a short overview about $\OPRA_m$ calculus.
We start this with a definition for a coarse type ($ m=2 $),
followed by the model for arbitrary $ m \in \mathbb{N}$.
Then we will present a new compact algorithm 
which to performs $\OPRA_m$ reasoning based on simple geometric rules,
and prove its correctness.
At the end we give an overview about several application that use the $ \OPRA_m $ calculus for 
spatial navigation simulations and discuss the adequateness of specific
choices for the granularity parameter $m$.

\section{The oriented point algebra}\label{sec:OPRA}

Objects and locations can be represented as simple, featureless points.
In contrast, the $ \OPRA_m $ calculus
uses more complex basic entities:
It is based on objects which are represented as oriented points.  It
is related to a calculus which is based on straight line segments
(dipoles) \cite{Moratz00_QSRwithLineSegs}.  Conceptually, the oriented points 
can be viewed as a transition from oriented line segments
with concrete length to line segments with infinitely small length
\cite{MoratzEtAl2010}.
In this conceptualization the length of the objects no longer has any
importance.  Thus, only the orientation of the objects is modeled.  {\em
  O-points}, our term for oriented points, are specified as pair of a
point and a orientation on the 2D-plane.

\begin{figure}[htb]
\begin{center}
\includegraphics[width=4.5cm]{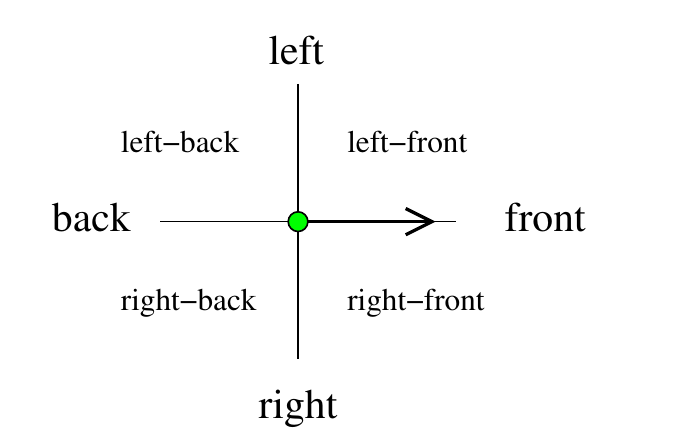}
\caption{\label{Regions} An oriented point and its qualitative spatial relative directions}
\end{center}
\end{figure}

\subsection{Qualitative O-Point Relations and Reasoning\label{basic}}

\noindent In a coarse representation a single o-point induces the
sectors depicted in \nolinebreak{figure \ref{Regions}}.  
``front'', 
``back'',
``left'', and ``right'' 
are linear sectors. 
``left-front'', ``right-front'', 
``left-back'', and ``right-back''
are quadrants.
The position of the point itself is denoted as ``same''.  
This qualitative granularity corresponds to Freksa's double\ednote{single cross?} cross calculus
\cite{freksa92b,ScivosN-a:04-finest}.

A qualitative spatial relative direction relation between two o-points
is represented by two pieces of information:
\begin{itemize}
\item the sector (seen from the first o-point) in which the second o-point lies (this
determines the lower part of the relation symbol), and
\item the sector (seen from the second o-point) in which the first o-point lies
(this determines the upper part of the relation symbol).
\end{itemize}
For the general case of the two points having different positions we use 
the following relation symbols:\\[-0.3cm]

${}_{\rm front}^{\rm front}$,
${}_{\rm front}^{\rm lf}$,
${}_{\rm front}^{\rm left}$,
${}_{\rm front}^{\rm lb}$,
${}_{\rm front}^{\rm back}$,
${}_{\rm front}^{\rm rb}$,
${}_{\rm front}^{\rm right}$,
${}_{\rm front}^{\rm rf}$,
${}_{\rm lf}^{\rm front}$,
${}_{\rm lf}^{\rm lf}$,
$\ldots$,
${}_{\rm rf}^{\rm rf}$.\\[-0.3cm]

\noindent 
Altogether we obtain $8 \times 8$ base relations for the two points having 
different positions.

Then the configuration shown in
\nolinebreak{figure \ref{Configuration}}
is expressed with the relation $A \; {}_{\rm rf}^{\rm lf} \; B$. If both points share the same
position, the lower relation symbol part is the word ``same'' and the upper part denotes
the direction of the second o-point with respect to the first one as shown in figure \ref{Configuration2}.

\begin{figure}[htb]
\begin{center}
\includegraphics[width=3.5cm]{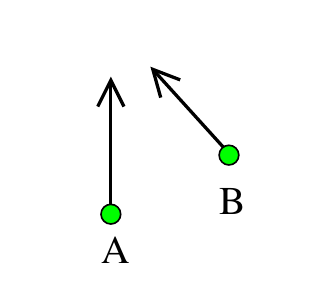}
\caption{\label{Configuration} Qualitative spatial relation between two oriented points at different positions.
The qualitative spatial relation depicted here is $A \; {}_{\rm rf}^{\rm lf} \; B$.}
\end{center}
\end{figure}

\noindent Altogether we obtain 72 different atomic
relations (eight times eight general relations plus eight with the
o-points at the same position). These relations are jointly exhaustive
and pairwise disjoint (JEPD).  The relation ${}_{\rm same}^{\rm front}$ is the
identity relation.  

\begin{figure}[htb]
\vspace*{-0.4cm}
\begin{center}
\includegraphics[width=2.8cm]{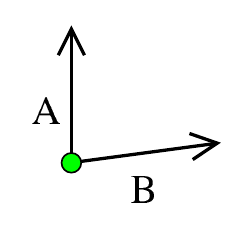}
\vspace*{-0.1cm}\caption{\label{Configuration2} Qualitative spatial relation between two
  oriented points located at the same position.
The qualitative spatial relation depicted here is $A \; {}_{\rm same}^{\rm rf} \; B$.}
\end{center}
\end{figure}

In order to apply constraint-based reasoning
to a set of qualitative spatial relations, the relations ideally should
form a relation algebra \cite{Ladkin94_RelAlg} or a non-associative
algebra \cite{Maddux2006,ligozatR-a:04-what}.  Such an algebra can be
generated from a jointly exhaustive and pairwise disjoint set of base
relations by forming the power set, giving the general relations,
with bottom, top, intersection, union and complement of relations
defined in the set-theoretic way.  Moreover, an identity base relation
and a converse operation ($\smile$) on base relations must be
provided; the latter naturally extends to general relations.  Finally,
if composition of base relations cannot be expressed using general
relations (strong composition), this operation
is approximated by a \emph{weak composition} \cite{RenzLigozat}:
$$b_1\diamond b_2=\{b \mid (R_{b_1}\circ R_{b_2})\cap R_b \not= \emptyset\}$$
where $R_{b_1}\circ R_{b_2}$ is the usual set theoretic composition
$$R\circ S = \{(x,z) | \exists y\,.\,(x,y)\in R, (y,z)\in S\}$$
and $R_b$ is the set-theoretic relation corresponding to the
abstract base relation $b$.
For details we refer to \cite{ligozatR-a:04-what}.

The composition of relations must be computed based on the semantics
of the relations. The compositions are usually computed only for the
atomic relations; this information is stored in a composition table.
The composition of compound relations can be obtained as the union of
the compositions of the corresponding atomic relations.  The
compositions of the atomic relations can be deduced directly from the
geometric semantics of the relations (see section
\ref{OPRACompTable}).

O-point constraints are written as $xRy$ where $x,y$ are variables for
o-points and $R$ is a $\OPRA$ relation.
Given a set $\Theta$ of o-point constraints, an important reasoning
problem is deciding whether $\Theta$ is {\em consistent}, i.e.,
whether there is an assignment of all variables of $\Theta$ with
dipoles such that all constraints are satisfied (a {\em solution}).
A partial method for determining inconsistency of a set of constraints
$\Theta$ is the
{\em path-consistency method} \cite{mackworth77}, which 
computes the algebraic closure on
$\Theta$.
This method applies
the following operation
until
a fixed point is reached:
\begin{displaymath}
\forall i, j, k:  \quad  R_{ij}  \leftarrow R_{ij} \cap (R_{ik} \diamond R_{kj})
\end{displaymath}
where $i,j,k$ are nodes and $R_{ij}$ is the relation between $i$ and $j$.
The resulting set of constraints is equivalent to the original set, i.e.\ it has the same set
of solutions. If the empty relation occurs while performing this operation, $\Theta$
is inconsistent, otherwise the resulting set algebraically closed\footnote{which means that it is path-consistent in the case that the algebra
has a strong composition} \cite{RenzLigozat}. 
Note that algebraic closure not always implies consistency, and
indeed, 
\cite{FrommbergerEtAl07} show
that this implication
does not hold for the
$\OPRA$ calculus.
Indeed, consistency in $\OPRA$ has been shown to be NP-hard
even for scenarios in base relations \cite{LeeWolterAIJ}, while algebraic closure
is a polynomial approximation of consistency.

\subsection{Finer Grained O-Point Calculi}

\noindent The design principle for
the coarse $\OPRA$ calculus described above
can be generalized to calculi
$\OPRA_{m}$ with arbitrary $ m \in \mathbb{N} $.
Then an angular resolution of
$\frac{2\pi}{2m}$
is used for the representation
(a similar scheme for absolute direction instead of relative direction was
designed by Renz and Mitra
\cite{Renz04_QDCArbGranu}).
The granularity used for the introduction of the $\OPRA$ calculus in the previous
section is $m = 2$, the corresponding $\OPRA$ version is then called $\OPRA_{2}$.

To formally specify the o-point relations we use two-dimensional
continuous space, in particular ${\mathbb{R}}^2$.  Every o-point $S$
on the plane is an ordered pair of a point ${\bf p}_S$ represented by
its Cartesian coordinates $x$ and $y$, with $x, y \in {\mathbb{R}}$
and an orientation $\phi_S$.

\begin{displaymath}
S = \left( {\bf p}_S, \phi_S \right) ,   \qquad
{\bf p}_S  = \left( ({\bf p}_S)_x , ({\bf p}_S)_y \right)
\end{displaymath}

We distinguish the relative locations and directions of the two
o-points $A$ and $B$ expressed by a calculus $\OPRA_{m}$ according to
the following scheme.  For $A$, $B$ with ${\bf p}_A\not={\bf p}_B$, we define
\[
\varphi_{AB} := \mathit{atan2}(({\bf p}_B)_y - ({\bf p}_A)_y, ({\bf p}_B)_x - ({\bf p}_A)_x)
\]
where $\mathit{atan2}(y, x)$ is the angle between the positive
$x$-axis and the point $(x, y)$, normalised to the interval $]-\pi, 
\pi]$. By the properties of $\mathit{atan2}$, we get
\[
\varphi_{BA} = \varphi_{AB}+\pi
\]
modulo normalization to $]-\pi, \pi]$.
 In the sequel, we will normalize all angles to this interval,
reflecting the cyclic order of the directions.
Hence, e.g.\ $-\pi$ stands for $\pi$. Moreover, in case
that $\alpha>\beta$, the open interval $]\alpha,\beta[$ will stand for
$]\alpha,\pi]\,\cup\,]-\pi,\beta[$. For example,
$]\frac{\pi}{2},-\frac{\pi}{2}[$ stands for
$]\frac{\pi}{2},\pi]\,\cup\,]-\pi,-\frac{\pi}{2}[$.

Similarly, we enumerate directions by using the $4m$ elements of the
cyclic group ${\cal Z}_{4m}$.  Each element of the cyclic group is
interpreted as a range of angles as follows:
\[
[i]_m = \left\{\begin{array}{ll}
{}] 2 \pi \frac{i - 1}{4 m}, 2 \pi \frac{i+1}{4 m}[, & \mbox{ if $i$ is odd}\\
\{ 2 \pi \frac{i}{4 m}\},  & \mbox{ if $i$ is even}
\end{array}\right.
\]
Conversely, for each angle $\alpha$, there is a unique element
$i\in{\cal Z}_{4m}$ with $\alpha\in[i]_m$.

\noindent 
If ${\bf p}_A \not= {\bf p}_B$, the relation $A \; {\scriptscriptstyle
  m}\angle_{i}^{j} \; B$ ($i,j \in {\cal Z}_{4m}$) reads like this: Given a
granularity $m$, the relative position of B with respect to A is described by
$i$ and the relative position of A with respect to B is described by $j$.
Formally, it
represents the  set of configurations satisfying
\begin{eqnarray*}
\label{angulardefinition}
&
\varphi_{AB}
- \phi_A \in [i]_m
\mbox{ ~~and~~ } 
\varphi_{BA}
- \phi_B \in[j]_m.
 \nonumber
\end{eqnarray*}

\begin{figure}[htb]
\vspace{-0.4cm}\begin{center}
\includegraphics[width=6.0cm]{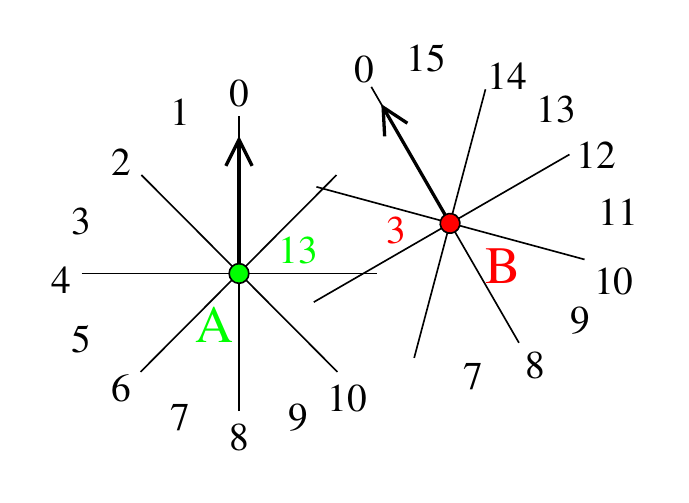}
\vspace{-0.4cm}\caption{\label{fig:OPRA4_Example}
Two o-points in relation
$A \; {\scriptscriptstyle 4}\angle_{13}^{3} \; B$
}
\end{center}
\end{figure}

\noindent 
Figure \ref{fig:OPRA4_Example} shows the resulting granularity
for $m=4$.
Using this notation, a simple manipulation of the parameters
yields
the converse operation
$({\scriptscriptstyle m}\angle^{i}_{j})^{\smile} = {\scriptscriptstyle m}\angle^{j}_{i}$
\hspace*{0.1cm}.

If ${\bf p}_A = {\bf p}_B$, the relation
$A \; {\scriptscriptstyle m}\angle{i} \; B$
represents the set of configurations satisfying
\begin{eqnarray*}
&
\phi_B - \phi_A \in [i]_m.
 \nonumber
\end{eqnarray*}

Hence the relation for two identical o-points $ A=B $ for arbitrary $m \in \mathbb{N} $
is $A{\scriptscriptstyle m}\angle{0}B$.
\noindent Using this notation a simple manipulation of the parameters
yields
the converse operation
$({\scriptscriptstyle m}\angle{i})^{\smile} = {\scriptscriptstyle m}\angle{(4m - i)}$.
The composition tables
for the atomic relations
of the $\OPRA_{m}$ calculi
can be computed using a small set of 
 simple formulas detailed in the following subsection.

It should be mentioned that the passage from $\OPRA_{1}$
to $\OPRA_{m}$ ($m\geq 2$) is a qualitative jump:
while $\OPRA_{1}$ relations are preserved by all orientation-preserving
affine bijections, for $m\geq 2$, $\OPRA_{m}$ relations
are only preserved by all angle-preserving
affine bijections, see \cite{MoratzEtAl2010}.


%
%

\begin{proposition}
Composition in $\OPRA$ is weak.
\end{proposition}

\begin{proof}
  The configuration $A\opra{1}{0}{0}B$, $B\opra{1}{1}{2}C$ and
  $A\opra{1}{3}{3}C$ is realizable. However, given $A$ and $C$ as in
  Fig.~\ref{fig:OPRA_weak}, we have $A\opra{1}{3}{3}C$, but we cannot
  find $B$ with $A\opra{1}{0}{0}B$ and $B\opra{1}{1}{2}C$: by
  $A\opra{1}{0}{0}B$, $B$'s carrier line is the same as $A$'s, and the
  two o-points face each other. But then, $B\opra{1}{1}{2}C$ is not
  possible, since $B$ would have to be located in the back of $C$.\ednote{RM: introduce in Fig. 5
  the direction of $C$ making $B\opra{1}{1}{2}C$ possible (i.e. $B$ is
  in the back of $C$) with a dashed line}

The argument easily generalizes to $\OPRA_m$ by considering
$A\opra{m}{0}{0}B$, $B\opra{m}{1}{2m}C$ and
$A\opra{m}{4m-1}{4m-1}C$.
\begin{figure}[htb]
\begin{center}
\includegraphics{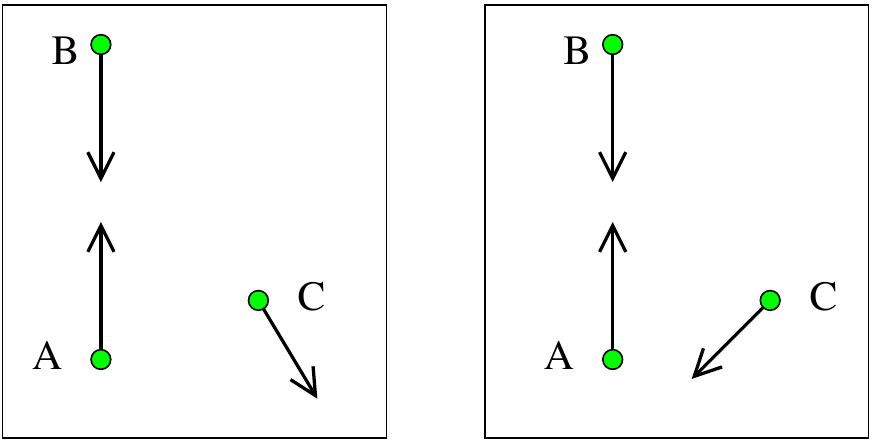}
\caption{\label{fig:OPRA_weak}
Composition in $\OPRA$ is weak}
\end{center}
\end{figure}
\end{proof}
\ednote{Adapt Fig.~\ref{fig:OPRA_weak} to the style of the others. RM: Is this figure correct?
see email.}

\subsection{Simple geometric rules for reasoning in $\OPRA_{m}$}
\label{OPRACompTable}

The composition table can be viewed as a list (set) of all relation
triples $A r_{ab} B$, $B r_{bc} C $, $C r_{ca} A$ for which $r_{ab},
r_{bc}$, and $r_{ca}$ are consistent ($A$, $B$, and $C$ being
arbitrary o-points on the $\mathbb{R}^2$ plane).  In the literature,
there are two algorithms for computing the composition table:
\cite{DBLP:conf/ecai/Moratz06} presents a fairly simple algorithm,
which, however, is error-prone, and \cite{FrommbergerEtAl07} provide a
correct algorithm, which however is based on a complicated case
distinction with dozens of cases (the paper is 29 pages long, 22 of
which are devoted to the algorithm and its correctness!).

We give an algorithm that is both correct and simpler than the two
existing algorithms.

The first ingredient of the algorithm is a detection of complete turns.
We define
\[
turn_m(i,j,k) \mbox{ iff } i+j+k \in  \left\{ 
\begin{array}{ll}
\{-1, 0, 1\}, &\mbox{ if both $i$ and $j$ are odd}\\
\{0\}, & \mbox{ otherwise }
\end{array}\right.
\]

This definition determines complete turns in the following sense:
\begin{proposition}\label{turn}
\begin{enumerate}
\item
$turn_m(i,j,k)\mbox{ iff }~\exists \alpha\in[i]_m, \beta\in[j]_m, \gamma\in[k]_m\,. \, \alpha+\beta+\gamma=0$
\item
$turn_m(i,j,k)$ implies that for any choice of one of the three angles 
in its interval, a suitable choice for the other two exists such that
all three add up to $0$.
\end{enumerate}
Recall that angles are normalized into $]-\pi,\pi]$.
\end{proposition}
\begin{proof}
We prove the first statement by a case distinction.
Case 1: both $i$ and $j$ are even.
This means that $[i]_m=\{ 2 \pi \frac{i}{4 m}\}$ and 
$[j]_m=\{ 2 \pi \frac{j}{4 m}\}$. Hence,
\[
\begin{array}{ll}
& \exists \alpha\in[i]_m, \beta\in[j]_m, \gamma\in[k]_m\,. \, \alpha+\beta+\gamma=0\\
\mbox{iff} &
\exists \gamma\in[k]_m\,.\,  2 \pi \frac{i+j}{4 m}+\gamma=0\\
\mbox{iff} &
i+j+k=0\\
\mbox{iff} &
turn_m(i,j,k)
\end{array}
\]

Case 2: $i$ is odd and $j$ is even.
This means that $[i]_m=] 2 \pi \frac{i - 1}{4 m}, 2 \pi \frac{i+1}{4 m}[$ and 
$[j]_m=\{ 2 \pi \frac{j}{4 m}\}$. Hence,
\[
\begin{array}{ll}
& \exists \alpha\in[i]_m, \beta\in[j]_m, \gamma\in[k]_m\,. \, \alpha+\beta+\gamma=0\\
\mbox{iff} &
\exists \gamma\in[k]_m\,.\,  -2 \pi \frac{j}{4 m}-\gamma\in] 2 \pi \frac{i - 1}{4 m}, 2 \pi \frac{i+1}{4 m}[ \\
\mbox{iff} &
\exists \gamma\in[k]_m\,.\,  \gamma\in] 2 \pi \frac{-i-j - 1}{4 m}, 2 \pi \frac{-i-j+1}{4 m}[ \\
\mbox{iff} &
\exists \gamma\in[k]_m\,.\,  \gamma\in[-i-j]_m \\
\mbox{iff} &
k=-i-j\\
\mbox{iff} &
turn_m(i,j,k)
\end{array}
\]

Case 3: $i$ is even and $j$ is odd: analogous to case 2.

Case 4: both $i$ and $j$ are odd.
This means that $[i]_m=] 2 \pi \frac{i - 1}{4 m}, 2 \pi \frac{i+1}{4 m}[$ and 
$[j]_m=] 2 \pi \frac{j - 1}{4 m}, 2 \pi \frac{j+1}{4 m}[$ .
Hence,
\[
\begin{array}{ll}
& \exists \alpha\in[i]_m, \beta\in[j]_m, \gamma\in[k]_m\,. \, \alpha+\beta+\gamma=0\\
\mbox{iff} &
\exists \gamma\in[k]_m\,.\,  \gamma\in] 2 \pi \frac{-i-j - 2}{4 m}, 2 \pi \frac{-i-j+2}{4 m}[ \\
\mbox{iff} &
\exists \gamma\in[k]_m\,.\,  \gamma\in [-i-j-1]_m \cup [-i-j]_m \cup [-i-j+1]_m  \\
\mbox{iff} &
k\in\{-i-j-1,-i-j,-i-j+1\}\\
\mbox{iff} &
turn_m(i,j,k)
\end{array}
\]

The second statement is straightforward when inspecting the proof above.
\end{proof}

Next, we turn to triangles. In a triangle, the sum of angles is always
$\pi$. Moreover, all angles have the same sign. We include the
degenerate case where two angles are 0 and the remaining one is $\pi$
(this corresponds to three points on a line), but we exclude the case
of three angles being $\pi$ (this is not geometrically realizable).
This leads to the following definitions:
\[
\begin{array}{l}
sign_m(i) = \left\{ 
\begin{array}{ll}
  0, & \mbox{if } (i\mod 4m=0) \vee (i\mod 4m=2m)\\ 
  1, & \mbox{if } i\mod 4m<2m\\
  -1, & \mbox{otherwise}
\end{array}
\right. \\[4ex]
triangle_m(i,j,k) \mbox{ iff } \\
\qquad  turn_m(i,j,k-2m) \wedge 
  (i,j,k) \not= (2m,2m,2m) \wedge
  sign_m(i)=sign_m(j)=sign_m(k) \\[2ex]
\end{array}
\]
Here, the angle $\pi$ also has sign $0$, which corresponds to the
geometric intuition and to the fact that the choice between $-\pi$
and $\pi$ to represent this angle is rather arbitrary.

From the above discussion, it is then straightforward to see
\begin{proposition}\label{triangle}
\[
triangle_m(i,j,k) \mbox{ iff }\exists \alpha\in[i]_m, \beta\in[j]_m, \gamma\in[k]_m\,. \,\mbox{ there exists a triangle with angles }\alpha, \beta, \gamma 
\]
\end{proposition}

Algorithm~\ref{OPRAcomp} now gives the complete algorithm for
computing $\OPRA_m$ compositions.  Note that we have slightly
rephrased the definition of $turn_m(i,j,k)$, the new version already
taking care of our convention regarding the cyclic group ${\cal
  Z}_{4m}$ and thus being directly implementable as a computer program
using the usual integers instead of ${\cal Z}_{4m}$.

\begin{algorithm}
~\\
$\begin{array}{l}
turn_m(i,j,k) \mbox{ iff }  |(i+j+k+2m)\mod 4m)-2m| \leq (i\mod 2) \times (j\mod 2)\\[2ex]
sign_m(i) = \left\{ 
\begin{array}{ll}
  0, & \mbox{if } (i\mod 4m=0) \vee (i\mod 4m=2m)\\ 
  1, & \mbox{if } i\mod 4m<2m\\
  -1, & \mbox{otherwise}
\end{array}
\right. \\[4ex]
triangle_m(i,j,k) \mbox{ iff } \\
\qquad  turn_m(i,j,k-2m) \wedge 
  (i,j,k) \not= (2m,2m,2m) \wedge
  sign_m(i)=sign_m(j)=sign_m(k) \\[2ex]
opra(\opras{m}{i},\opras{m}{k},\opras{m}{s})\mbox{ iff } turn_m(i,k,-s)\\
opra(\opras{m}{i},\opras{m}{k},\opra{m}{s}{t})\mbox{ iff } \false\\
opra(\opras{m}{i},\opra{m}{k}{l},\opras{m}{s})\mbox{ iff } \false\\
opra(\opras{m}{i},\opra{m}{k}{l},\opra{m}{s}{t})\mbox{ iff } l=t \wedge turn_m(i,k,-s)\\
opra(\opra{m}{i}{j},\opras{m}{k},\opras{m}{s})\mbox{ iff } \false\\
opra(\opra{m}{i}{j},\opras{m}{k},\opra{m}{s}{t})\mbox{ iff } i=s \wedge turn_m(t,k,-j)\\
opra(\opra{m}{i}{j},\opra{m}{k}{l},\opras{m}{s})\mbox{ iff } j=k \wedge turn_m(i,-l,-s)\\
opra(\opra{m}{i}{j},\opra{m}{k}{l},\opra{m}{s}{t})\mbox{ iff } \\
~~~~\exists 0\leq u,v,w< 4m\, .\,  turn_m(u,-i,s) \wedge turn_m(v,-k,j) \wedge turn_m(w,-t,l)  \wedge triangle_m(u,v,w) \\
\end{array}$
\caption{Checking entries of the $\OPRA_m$ composition table}
\label{OPRAcomp}
\end{algorithm}

\begin{theorem}
Algorithm~\ref{OPRAcomp} computes composition in $\OPRA_m$.
\end{theorem}
\begin{proof}

Case 
$opra(\opras{m}{i},\opras{m}{k},\opras{m}{s})$.
Since the points of all o-points are the same, their direction
must add up to a complete turn.
More precisely, the configuration $\opras{m}{i},\opras{m}{k},\opras{m}{s}$
is realizable iff there are o-points $A$, $B$ and $C$ with
${\bf p}_A = {\bf p}_B = {\bf p}_C$, 
$\phi_B-\phi_A\in[i]_m$, 
$\phi_C-\phi_B\in[k]_m$, and
$\phi_C-\phi_A\in[s]_m$. Since for such $A$, $B$ and $C$,
$(\phi_B-\phi_A)+(\phi_C-\phi_B)-(\phi_C-\phi_A)=0$ (i.e.\ we have 
a complete turn), by Proposition~\ref{turn} this is in turn equivalent
to  $turn_m(i,k,-s)$.
\ednote{draw picture. RMO: if space restrictions permit. TM: this should
be there.}
\begin{figure}[htb]
\begin{center}
\includegraphics{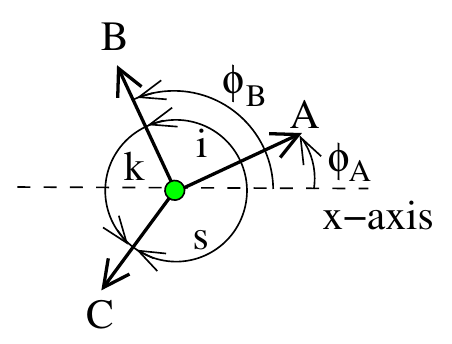}
\caption{\label{fig:turnIKS}
${\bf p}_A = {\bf p}_B = {\bf p}_C$
}
\end{center}
\end{figure}

Cases 
$opra(\opras{m}{i},\opras{m}{k},\opra{m}{s}{t})$, $opra(\opras{m}{i},\opra{m}{k}{l},\opras{m}{s})$ and
$opra(\opra{m}{i}{j},\opras{m}{k},\opras{m}{s})$.
Since sameness of points is transitive, these cases are not
realizable.

Cases $opra(\opras{m}{i},\opra{m}{k}{l},\opra{m}{s}{t})$,
$opra(\opra{m}{i}{j},\opras{m}{k},\opra{m}{s}{t})$ and
$opra(\opra{m}{i}{j},\opra{m}{k}{l},\opras{m}{s})$.  We here only
treat the case $opra(\opras{m}{i},\opra{m}{k}{l},\opra{m}{s}{t})$;
the other cases being analoguous.
The configuration $A \opras{m}{i} B, B\opra{m}{k}{l}C, A\opra{m}{s}{t}C$
is realizable iff 
\[
\left.
\begin{array}{l}
\mbox{there are o-points $A$, $B$ and $C$ with}\\
~~~{\bf p}_A = {\bf p}_B, 
\phi_B-\phi_A\in[i]_m,\\
~~~\phi_{BC}-\phi_B\in[k]_m, \phi_{CB}-\phi_C\in[l]_m,\\
~~~\phi_{AC}-\phi_A\in[s]_m\mbox{ and }
\phi_{CA}-\phi_C\in[t]_m.
\end{array}
\right\}(*)
\]
We now show that $(*)$ is equivalent to 
\[l=t\mbox{ and }turn_m(i,k,-s).\]
Assume $(*)$. By ${\bf p}_A = {\bf p}_B$, we have $\phi_{BC}=\phi_{AC}$
and $\phi_{CB}=\phi_{CA}$; from the latter, we also get $l=t$.
Moreover, $(\phi_B-\phi_A)+(\phi_{BC}-\phi_B)-(\phi_{AC}-\phi_A)=0$ is
a turn, and by Proposition~\ref{turn}, we get $turn_m(i,k,-s)$.
Conversely, assume $l=t$ and $turn_m(i,k,-s)$.
By Proposition~\ref{turn}, there are angles $\alpha,\beta,\gamma$ with
$\alpha\in[i]_m$, $\beta\in[k]_m$ and $\gamma\in[-s]_m$.
Choose $A$ arbitrarily. Then define $B$ by
${\bf p}_B = {\bf p}_A$ and $\phi_B=\alpha-\phi_A$.
Then choose ${\bf p}_C$ on the half-line starting from ${\bf p}_A$
and having angle $\beta$ to $B$ and $-\gamma$ to $A$.
Finally, choose $\phi_C$ such that
$\phi_{CA}-\phi_C=\phi_{CB}-\phi_C\in[t]_m=[l]_m$. This
ensures the conditions of $(*)$.
\ednote{draw picture. RMO: if space restrictions permit. TM: this picture
should be there.}
\begin{figure}[htb]
\begin{center}
\includegraphics{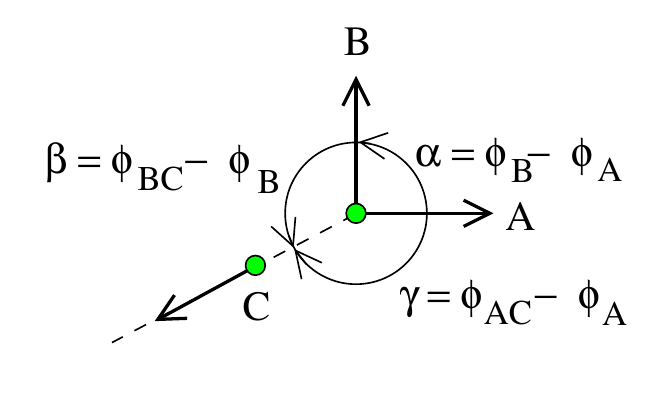}
\caption{\label{fig:phiC}
${\bf p}_A = {\bf p}_B \neq {\bf p}_C$
}
\end{center}
\end{figure}

Case 
$opra(\opra{m}{i}{j},\opra{m}{k}{l},\opra{m}{s}{t})$.
We need to show that the existence of a configuration 
$A\opra{m}{i}{j}B$, $B\opra{m}{k}{l}C$ and
$A\opra{m}{s}{t}C$ is equivalent to
\[
\left.
\begin{array}{l}
\exists 0\leq u,v,w< 4m . \\
~~ turn_m(u,-i,s) \wedge turn_m(v,-k,j) \wedge turn_m(w,-t,l)\\
~~  \wedge triangle_m(u,v,w) 
\end{array}
\right\}(**)
\]
Given $A\opra{m}{i}{j}B$, $B\opra{m}{k}{l}C$ and
$A\opra{m}{s}{t}C$,
let $\alpha$, $\beta$ and $\gamma$ be the angles
of the triangle ${\bf p}_A{\bf p}_B{\bf p}_C$,
that is,
\[\begin{array}{l}
\alpha=\phi_{AB}-\phi_{AC}\\
\beta=\phi_{BC}-\phi_{BA}\\
\gamma=\phi_{CA}-\phi_{CB}
\end{array}
\]
\begin{figure}[!ht]
\vspace{-0.4cm}\begin{center}
\includegraphics[width=5.0cm]{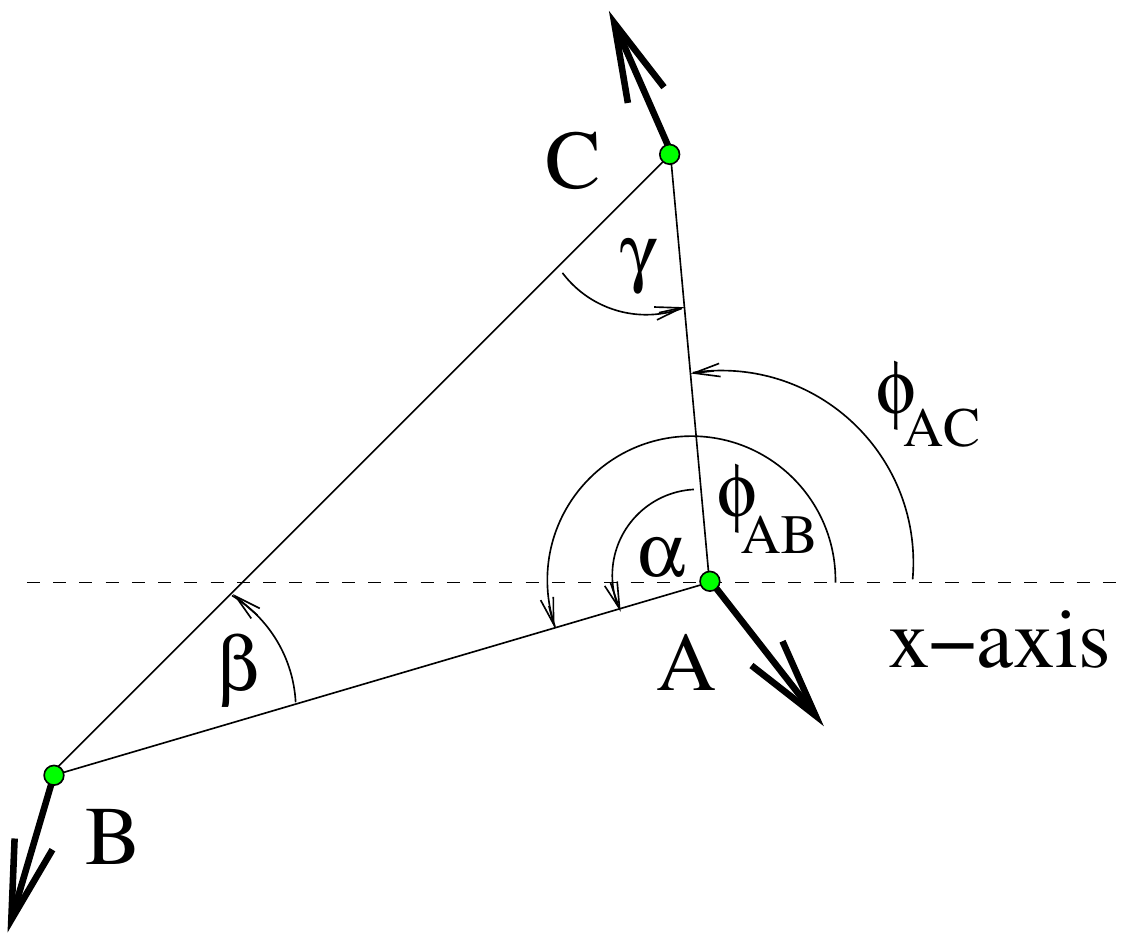}
\vspace{-0.2cm}\caption{The sum of angles in a triangle equals $\pi$} 
\label{fig:triangle} 
\end{center} 
\end{figure}
\ednote{adapt triangle to the present notation. RMO: if space restrictions permit,
otherwise drop the picture.}

Let $u, v, w\in {\cal Z}_{4m}$ be such that $\alpha\in[u]_m$,
$\beta\in[v]_m$ and $\gamma\in[w]_m$.  By Proposition~\ref{triangle},
$triangle_m(u,v,w)$.  At the corners of the triangle 
${\bf p}_A{\bf p}_B{\bf p}_C$, the following complete turns can be formed:
\begin{itemize}
\item $(\phi_{AB}-\phi_{AC})-(\phi_{AB}-\phi_A)+(\phi_{AC}-\phi_A)$, 
corresponding to $turn_m(u,-i,s)$ by Proposition~\ref{turn},
\item $(\phi_{BC}-\phi_{BA})-(\phi_{BC}-\phi_B)+(\phi_{BA}-\phi_B)$, 
corresponding to $turn_m(v,-k,j)$,
\item $(\phi_{CA}-\phi_{CB})-(\phi_{CA}-\phi_C)+(\phi_{CB}-\phi_C)$, 
corresponding to $turn_m(w,-t,l)$.
\end{itemize}
This shows $(**)$. Conversely, assume $(**)$.
By $triangle_m(u,v,w)$ and Proposition~\ref{triangle},
we can choose ${\bf p}_A$, ${\bf p}_B$ and ${\bf p}_C$ such that
\[\begin{array}{l}
\phi_{AB}-\phi_{AC}\in[u]_m\\
\phi_{BC}-\phi_{BA}\in[v]_m\\
\phi_{CA}-\phi_{CB}\in[w]_m
\end{array}
\]
Since $turn_m(u,-i,s)$, by Proposition~\ref{turn}, we can find $\alpha_A$, $\beta_A$ and $\gamma_A$
such that $\alpha_A+\beta_A+\gamma_A=0$ and 
$\alpha_A\in[-i]_m$, $\beta_A\in[s]_m$ and $\gamma_A\in[u]_m$.
From Proposition~\ref{turn}(2), we obtain that it is
possible to choose $\gamma_A=\phi_{AB}-\phi_{AC}$ (note that the latter
angle is also in $[u]_m$). 
Put $\phi_A:=\phi_{AB}+\alpha_A$, then $\phi_{AB}-\phi_A=-\alpha_A\in[i]_m$,
and $\phi_{AC}-\phi_A=(\phi_{AB}-\phi_A)-(\phi_{AB}-\phi_{AC})=-\alpha_A-\gamma_A=\beta_A\in[s]_m$. $\phi_B$ and $\phi_C$ can be chosen similarly,
fulfilling the constraints given by $j$ and $k$ resp.\ $l$ and $t$.
\end{proof}

Using Algorithm~\ref{OPRAcomp}, a composition table for $\OPRA_{m}$
can be computed by enumerating all possible triples and only keeping
those for which the predicate $opra$ holds.  Moreover, given a pair of
$\OPRA_{m}$ relations, by enumerating all possible third $\OPRA_{m}$
relations and testing with the predicate $opra$, also the composition
of two relations can be computed.  

The run time of the predicate $opra$ is $O(m^3)$, since the algorithm
contains an existential quantification over the variables $u$, $v$,
$w$ ranging from $0$ to $4m-1$.  However, the existential
quantification can be replaced by a constant number of case
distinctions: e.g.\ we look for $u$ such that $turn_m(u,-i,s)$. But since
$u-i+s$ must add up to $-1$, $0$ or $1$, it is clear that $u$ must be
taken from the set $\{i-s-1,i-s,i-s+1\}$.  As a result, we get an
improved run time that is constant. This holds only when assuming a
  register machine with arithmetic operations executed in constant
  time. For a Turing machine with binary representations of numbers,
  basic arithmetic operations take time $\log m$. Then the run time is
  $O(\log m)$.

The computation of the composition of two relations needs to enumerate
all possible third relations, and then check each triple in constant
time.  Since there are $(4m)^2+4m$ relations, this takes $O(m^2)$
time, which is the same time as in \cite{FrommbergerEtAl07}.
Again, for Turing machines, this multiplies by a factor of $\log n$,
  hence we get an overall running time of $O(m^2\log m)$. Of course,
  the same remark applies to the algorithm of
  \cite{FrommbergerEtAl07}.

  A Haskell version of the $opra$ algorithm (also implementing the
  above optimization of the existential quantification) can be
  downloaded at \url{http://quail.rsise.anu.edu.au}.\ednote{fill in url. Should we
host the code on Google code, because this is more stable than
a URL involving our home pages? Aha, even better: the quail wiki http://quail.rsise.anu.edu.au}

The $\OPRA$ calculus can be used to express many other
qualitative position calculi \cite{cosy:dylla:2006:OPRA_GeneralFramework}.

\section{Applications of the $\OPRA$ calculus}

Spatial knowledge expressed in  
$\OPRA$
can be used for deductive reasoning based on constraint propagation (algebraic closure), resulting
in the generation of useful indirect knowledge from partial observations in a spatial scenario.
Several researchers developed applications using the $\OPRA$ calculus.
We will give a short overview and then make some concluding comments about
the first $\OPRA$ calculus applications in our conclusion section that follows.

\begin{figure}[htb]
\begin{center}
\centerline{
\includegraphics[height=4cm]{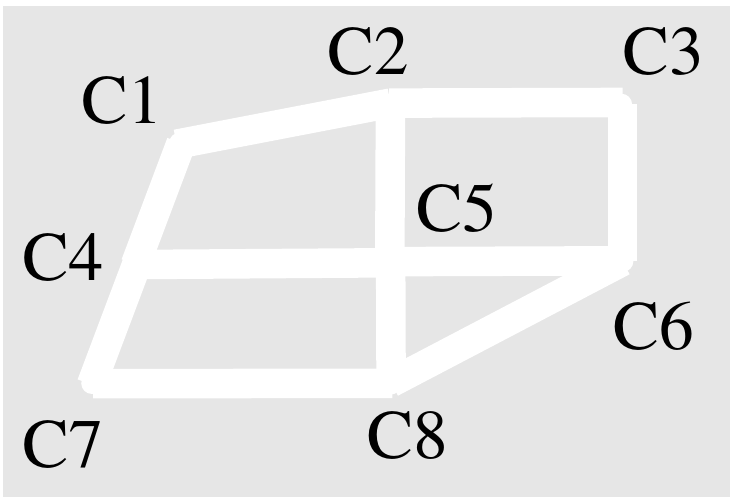}
\includegraphics[height=4cm]{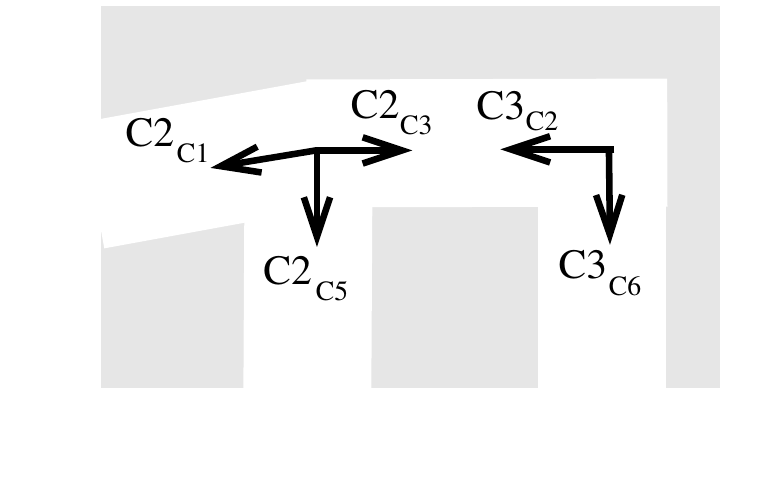}
}
\caption{\label{fig:streets_OPRA2}
Street networks with unique crossing names (detail with o-points to the right).
}
\end{center}
\end{figure}
\ednote{Relations would be nice here
RM: We have two relation examples from OPRA2
and four examples from OPRA4.
Which ones should be added? 
The whole street network would be too big and is very regular.
}

In a simple application by L{\"u}cke et. al. 
\cite{LueckeMoratzMossakowskiSFBreport2010} for benchmarking purposes between different spatial
calculi,
a spatial agent (a simulated robot, cognitive simulation of a
biological system etc.) explores a spatial scenario.
The agent collects local observations and wants to generate survey knowledge.
Fig.~\ref{fig:streets_OPRA2}
shows a spatial environment consisting of a street network. 
The notation $C2_{C1}$ refers to the o-point at position $C2$ with an intrinsic 
orientation towards
point $C1$. 
In this street network some streets continue straight after a crossing and
some streets meet with orthogonal angles. 
These features are typical of real-world street networks and can be directly represented
in $\OPRA_{2}$ expressions about o-points that constitute relative directions of
o-points located at crossings and pointing to neighboring (e.g. visible) crossings.
For example two relations corresponding to local observations referring to the
street network part depicted in Fig.~\ref{fig:streets_OPRA2} are:
$C2_{C3} \, {}_{\rm front}^{\rm front} \, C3_{C2}$
and 
$C2_{C3} \, {}_{\rm same}^{\rm right} \, C2_{C5}$.
Spatial reasoning in this spatial agent simulation 
uses constraint propagation (e.g. algebraic closure computation)
to derive indirect constraints between the relative location of streets which are
further apart from local observations between neighboring streets.
The resulting survey knowledge can be used for several tasks including navigation tasks.
The details of this scenario can be found in L{\"u}cke et. al. \cite{LueckeMoratzMossakowskiSFBreport2010}.

A related application developed by Wallgr{\"u}n \cite{wallgruenSCCaccepted}
uses Qualitative Spatial Reasoning
with $\OPRA$ to
determine the correct graph structure from a sequence of local observations by a simulated
robot collected while moving through an environment consisting of hallways.
These hallway networks are analogous to the street networks of L{\"u}cke et. al., but the
local observation are modeled in a more complex but more realistic way.
The identity of a crossing revisited after a cyclic path 
is not given but has to be inferred which makes navigation much more challenging.
Since there are many ambiguities left, the task is to track the multiple geometrically possible
topologies of the network during an incremental observation. 
Thus, the goal of Wallgr{\"u}n's qualitative mapping algorithm 
is to process the history of observations and 
determine all route graph hypotheses 
which can be considered valid explanations so far.
This consistency checking can be based on qualitative spatial reasoning about
positions. The local relative observation are modeled based on 
$\OPRA_{2}$ expressions about o-points 
in a similar way like in the street network
described above.

\begin{figure}[htb]
\begin{center}
\includegraphics[height=2.7cm]{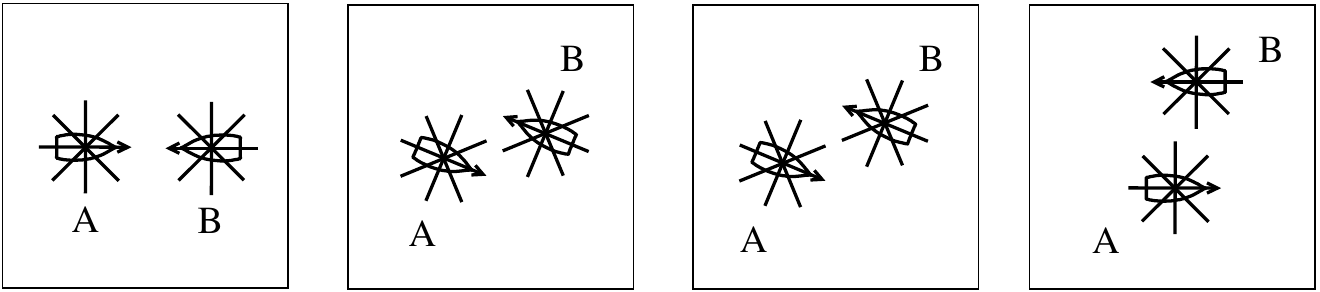}
\caption{\label{fig:SailAwayDiagramm}
Representation of vessel navigation with conceptual neighborhood in $\OPRA_{4}$
}
\end{center}
\end{figure}

A comprehensive simulation which uses the $\OPRA_{4}$ calculus for an important
subtask was built by Dylla et. al.
\cite{cosy:R3-R4-sailaway-aisb}.
Their system called SailAway simulates the behavior of different vessels following
declarative (written) navigation rules for collision avoidance. 
This system can be used to verify whether a given set of rules leads to stable
avoidance between potentially colliding vessels 
The different vessel categories that determine their right of way priorities are
represented in an ontology.
The movement of the vessels is described by a method called
conceptual neighborhood-based reasoning (CNH reasoning).
CNH reasoning describes whether two spatial configurations
of objects can be transformed into each other by small changes
\cite{cosyFRE91}, \cite{Galton00_QualSpatChange}.
A CNH transformation can be a object movement
in a short period of time.
Fig.~\ref{fig:SailAwayDiagramm} shows a CNH transition diagram which represents
relative trajectories of two rule following vessels.
The depicted sequence between two vessels $A$ and $B$ is:\\
$A \; {\scriptscriptstyle 4}\angle_{0}^{0} \; B \; \rightarrow
A \; {\scriptscriptstyle 4}\angle_{1}^{1} \; B \; \rightarrow
A \; {\scriptscriptstyle 4}\angle_{2}^{2} \; B \; \rightarrow
A \; {\scriptscriptstyle 4}\angle_{3}^{3} \; B\; $. 
Based on this
qualitative representation of trajectories, CNH 
reasoning is used as a simple, abstract model of the navigation
of the potentially colliding vessels in the SailAway simulator
\cite{cosy:R3-R4-sailaway-aisb})\footnote{An earlier version of qualitative 
navigation simulation by Dylla and Moratz can be found in \cite{DyllaMoratzSC04}}.

These three applications above 
make use of qualitative spatial
reasoning with\\ $\OPRA_{2}$ or $\OPRA_{4}$ 
in simulated spatial agent scenarios.
The granularity $m = 2$ can model straight continuation and right angles 
which are important for representing
idealized street networks.
The granularities $m = 2$ and $m = 4$ also correspond to earlier work about
computational models of linguistic projective expressions (left, right, in front, behind)
by Moratz et. al. \cite{MoratzFischerTenbrink} \cite{MoratzTenbrink2006}. 
The applications presented in this section could benefit from additional
qualitative relative distance knowledge. The TPCC calculus presented by Moratz \& Ragni
\cite{moratz05_QSR_RPP} is a first step towards this direction.
However, in contrast to our new results for the $\OPRA_m$ calculus 
the TPCC calculus only has a complex, manually derived and therefore
unreliable composition table. 

\vspace*{-0cm}\section{Summary and Conclusion}
We presented a calculus for representing and reasoning about
qualitative relative direction information.
Oriented points serve as the
basic entities since they are
the simplest spatial entities that
have an intrinsic orientation.
Sets of base relations can have adjustable granularity levels in this calculus.
We provided simple geometric rules
for computing the calculi's composition 
based on triples of oriented points.

We gave a short overview about three first applications 
that are based on oriented points and their relative position represented as $\OPRA_m$ relations
with granularity $m = 2$, or $m = 4$ which seem to be suited for
linguistically inspired spatial expressions. 

\section*{Acknowledgment} 
The authors would like to thank
Lutz Frommberger,
Frank Dylla,
Jochen Renz,
Diedrich Wolter,
Dominik L{\"u}cke,
and
Christian Freksa
for interesting and helpful
discussions related to the topic of the paper.
The work was supported by the DFG Transregional Collaborative
Research Center SFB/TR 8 {}``Spatial Cognition'' and by
the National Science Foundation under Grant No. CDI-1028895.


\bibliographystyle{plain}
\bibliography{KI05,oslsa}

\end{document}